%% file: main.tex
\def\year{2021}\relax
\documentclass[letterpaper]{article} 
\usepackage{aaai21}  
\usepackage{times}  
\usepackage{helvet} 
\usepackage{courier}  
\usepackage[hyphens]{url}  
\usepackage{graphicx} 
\usepackage{natbib}  
\usepackage{caption} 
\usepackage{bm}
\usepackage{amsfonts}
\usepackage{amsmath}
\usepackage{xcolor}
\usepackage{amsthm}
\usepackage{booktabs}
\usepackage{multirow}
\usepackage{subcaption}
\usepackage{enumitem}

\usepackage{csquotes}

\newtheorem{prop}{Proposition}

\usepackage{tikz}
\usetikzlibrary{arrows,arrows.meta}

\newcommand{\minD}{\texttt{Min-}$\delta$}
\newcommand{\maxD}{\texttt{Max-}$\delta$}
\newcommand{\minL}{\texttt{Min-}$U$}
\newcommand{\maxL}{\texttt{Max-}$U$}

\newcommand{\dom}{\texttt{Dom}}

\newcommand{\flipLogit}{\texttt{Logit}}

\DeclareMathOperator*{\argmax}{\arg\!\max}

\newcommand{\flipgroup}{\emph{Indecisive}}
\newcommand{\noflipgroup}{\emph{Strict}}

\title{Indecision Modeling}
\author {

        Duncan C McElfresh,\textsuperscript{\rm 1}
        Lok Chan,\textsuperscript{\rm 2}
        Kenzie Doyle,\textsuperscript{\rm 3} \\
        Walter Sinnott-Armstrong,\textsuperscript{\rm 2}
        Vincent Conitzer,\textsuperscript{\rm 2}
        Jana Schaich Borg,\textsuperscript{\rm 2}
        John P Dickerson\textsuperscript{\rm 1}\\
}
\affiliations {
    \textsuperscript{\rm 1} University of Maryland, College Park \\ 
    \textsuperscript{\rm 2} Duke University \\ 
    \textsuperscript{\rm 3} University of Oregon \\
    dmcelfre@umd.edu, lok.c@duke.edu
}

\begin{document}

\maketitle

\begin{abstract}

AI systems are often used to make or contribute to important decisions in a growing range of applications, including criminal justice, hiring, and medicine.
Since these decisions impact human lives, it is important that the AI systems act in ways which align with human values. 
Techniques for preference modeling and social choice help researchers learn and aggregate peoples' preferences, which are used to guide AI behavior; thus, it is imperative that these learned preferences are accurate.
These techniques often assume that people are willing to express strict preferences over alternatives; which is not true in practice.
People are often indecisive, and especially so when their decision has moral implications.
The philosophy and psychology literature shows that indecision is a measurable and nuanced behavior---and that there are several different reasons people are indecisive.
This complicates the task of both learning and aggregating preferences, since most of the relevant literature makes restrictive assumptions on the meaning of indecision.
We begin to close this gap by formalizing several mathematical \emph{indecision} models based on theories from philosophy, psychology, and economics; these models can be used to describe (indecisive) agent decisions, both when they are allowed to express indecision and when they are not.
We test these models using data collected from an online survey where participants choose how to (hypothetically) allocate organs to patients waiting for a transplant.

\end{abstract}

\section*{Introduction}\label{sec:intro}

AI systems are currently used to make, or contribute to, many important decisions.
These systems are deployed in self-driving cars, organ allocation programs, businesses for hiring, and courtrooms to set bail.
It is an ongoing challenge for AI researchers to ensure that these systems make decisions that align with human values.

A growing body of research views this challenge through the lens of \emph{preference aggregation}.
From this perspective, researchers aim to (1) understand the preferences (or values) of the relevant \emph{stakeholders}, and (2) design an AI system that aligns with the aggregated preferences of all stakeholders.
This approach has been proposed recently in the context of self-driving cars~\cite{noothigattu2018voting} and organ allocation~\cite{freedman2020adapting}.
These approaches rely on a mathematical model of stakeholder preferences--which is typically \emph{learned} using data collected via hypothetical decision scenarios or online surveys.\footnote{The MIT Moral Machine project is one example: \url{https://www.moralmachine.net/}}
There is a rich literature addressing how to elicit preferences accurately and efficiently, spanning the fields of computer science, operations research, and social science.

It is critical that these \emph{observed} preferences accurately represent peoples' \emph{true} preferences, since these observations guide deployed AI systems.
Importantly, the way we measure (or \emph{elicit}) preferences is closely tied to the accuracy of these observations.
In particular, it is well-known that both the order in which questions are asked, and the set of choices presented, impact expressed preferences~\cite{day2012ordering,deshazo2002designing}.

Often people choose \emph{not} to express a strict preference, in which case we call them \emph{indecisive}.
The economics literature has suggested a variety of explanations for indecision~\cite{gerasimou2018indecisiveness}---for example when there are no desirable alternatives, or when all alternatives are perceived as equivalent.
Moral psychology research has found that people often ``do not want to play god'' in moral situations, and would prefer for somebody or something else to take responsibility for the decision~\cite{gangemi2013moral}.

 In philosophy, indecision of the kind discussed in this paper is typically linked to a class of moral problems called symmetrical dilemmas, in which an agent is confronted with the choice between two alternatives that are or appear to the agent equal in value \cite{Sinnott-Armstrong1988-SINMD}.\footnote{Sophie's Choice is a well-known example: a guard at the concentration camp cruelly forces Sophie to choose one of her two children to be killed. The guard will kill both children if Sophie refuses to choose. Sophie’s reason for not choosing one child applies equally to another, hence the symmetry.} Much of the literature concerns itself with the morality and rationality of the use of a randomizer, such as flipping a coin, to resolve these dilemmas. Despite some disagreements over details \cite{ McIntyre1990-MACMD,Donagan1984-DONCIR, Blackburn1996-BLADDP,Hare1981-HARMTI}, many philosophers do agree that flipping a coin is often a viable course of action in response to indecision\footnote{With some exceptions: for example, see \cite{Railton1992-RAIPDA}.}. 
 
The present study accepts the assumption that flipping a coin is typically an expression of one's preference to \emph{not} decide between two options, but goes beyond the received view in philosophy by suggesting that indecision can also be common and acceptable when the alternatives are \emph{asymmetric}. 
We show that people often do adopt coin flipping strategies in asymmetrical dilemmas, where the alternatives are not equal in value. Thus, the use of a randomizer is likely to play a more complex role in moral decision making than simply as a tie breaker for symmetrical dilemmas.

%
%
%
Naturally, people are also sometimes indecisive when faced with difficult decisions related to AI systems. 
However it is commonly assumed in the preference modeling literature that people always express a strict preference, unless (A) the alternatives are approximately equivalent, or (B) the alternatives are incomparable.
Assumption (A) is mathematically convenient, since it is necessary for preference \emph{transitivity}.\footnote{My preferences are transitive if ``I prefer A over B'' and ``I prefer B over C'' implies ``I prefer A over C''.} 
Since indecision is both a common and meaningful response, strict preferences alone cannot accurately represent peoples' real values.
Thus, AI researchers who wish to guide their systems using observed preferences should be aware of the hidden meanings of indecision.
We aim to uncover these meanings in a series of studies. 

\noindent\textbf{Our Contributions.} First, we conduct a pilot experiment to illustrate how different interpretations of indecision lead to different outcomes (\S~\ref{sec:pilot}). Using hypothesis testing, we reject the common assumption (A) that indecision is expressed only toward equivalent---or symmetric---alternatives.

Then, drawing on ideas from psychology, philosophy, and economics, we discuss several other potential reasons for indecision, drawing (\S~\ref{sec:indecision}).
We formalize these ideas as mathematical indecision \emph{models}, and develop a probabilistic interpretation that lends itself to computation (\S~\ref{sec:indecision-models}). 

To test the utility of these models, we conduct a second experiment to collect a much larger dataset of decision responses (\S~\ref{sec:study2}).
We take a machine learning (ML) perspective, and evaluate each model class based on its goodness-of-fit to this dataset.
We assess each model class for predicting \emph{individual} peoples' responses, and then we briefly investigate group decision models.

In all of our studies, we ask participants  \emph{who should receive the kidney?} in a hypothetical scenario where two patients are in need of a kidney, but only one kidney is available. 
As a potential basis for their answers, participants are given three ``features'' of each patient: age, amount of alcohol consumption, and number of young dependents. 

We chose this task for several reasons: first, kidney exchange is a real application where algorithms influence---and sometimes make---important decisions about who receives which organ.\footnote{Many  exchanges match patients and donors algorithmically, including the United Network for Organ Sharing (\url{https://unos.org/transplant/kidney-paired-donation/}) and the UK national exchange (\url{https://www.odt.nhs.uk/living-donation/uk-living-kidney-sharing-scheme/}).}
Second, organ allocation is a \emph{difficult} problem: there are far fewer donors organs than there are people in need of a transplant.\footnote{There are around $100,000$ people in need of a transplant today (\url{https://unos.org/data/transplant-trends/}), and about $22,000$ transplants have been conducted in $2020$.} Third, the question of \emph{who} should receive these scarce resources raises serious ethical dilemmas~\cite{scheunemann2011ethics}.
Kidney allocation is also a common motivation for studies of fair resource allocation~\cite{agarwal2019empirical,mcelfresh2018balancing,mattei2018fairness}.
Furthermore, this type of scenario is frequently used to study peoples' preferences and behavior~\cite{freedman2020adapting,furnham2000decisions,furnham2002allocation,oedingen2018public}.
Importantly, this prior work focuses on peoples' \emph{strict} preferences, while we aim to study indecision.

\section*{Study 1: Indecision is Not Random Choice}\label{sec:pilot}
\input{pilot_study}

\section*{Models for Indecision}\label{sec:indecision}

The psychology and philosophy literature find several reasons for indecision, and many of these reasons can be approximated by numerical decision models.
Before presenting these models, we briefly discuss their related theories from psychology and philosophy.  

\noindent\textbf{Difference-Based Indecision}
In the preference modeling literature it is sometimes assumed that people are indecisive only when both alternatives (X and Y) are indistinguishable. 
That is, the perceived difference between X and Y is too small to arrive at a strict preference. In philosophy, this is referred to as ``the possibility of parity'' \cite{Chang2002-px}.

\noindent\textbf{Desirability-Based Indecision}
In cases where both alternatives are not ``good enough'', people may be reluctant to choose one over the other.
This has been referred to as ``single option aversion''~\cite{mochon2013single}, when consumers do not choose between product options if none of the options is sufficiently likable. 
\citet{zakay1984choose} observes this effect in single-alternative choices: people reject an alternative if it is not sufficiently close to a hypothetical ``ideal''.
Similarly, people may be indecisive if \emph{both} alternatives are attractive. People faced with the choice between two highly valued options often opt for an indecisive resolution in order to manage negative emotions \cite{Luce1998-je}.

\noindent\textbf{Conflict-Based Indecision}
People may be indecisive when there are both good and bad attributes of each alternative. 
This is phrased as \emph{conflict} by \citet{tversky1992choice}: people have trouble deciding between two alternatives if neither is better than the other in \emph{every way}.
In the AI literature, the concept of \emph{incomparability} between alternatives is also studied~\cite{pini2011incompleteness}. 

While these notions are intuitively plausible, we need mathematical definitions in order to model observed preferences. That  is the purpose of the next section.

\section*{Indecision Model Formalism}\label{sec:indecision-models}

In accordance with the literature, we refer to decision-makers as \emph{agents}.
Agent preferences are represented by binary relations over each pair of items $(i, j)\in \mathcal I \times \mathcal I$, where $\mathcal I$ is a universe of items.
We assume agent preferences are \emph{complete}: when presented with item pair $(i, j)$, they expresses exactly one response $r\in \{0, 1, 2\}$, which indicates:
\begin{itemize}
    \item $r=1$, or $i \succ j$: the agent prefers $i$ more than $j$
    \item $r=2$, or $i \prec j$: the agent prefers $j$ more than $i$
    \item $r=0$, or $i \sim j$: the agent is indecisive between $i$ and $j$
\end{itemize}
When preferences are complete and transitive,\footnote{Agent preferences are transitive if $i \succ j$ and $i \succ k$ iff $i \succ k$.} then the preference relation corresponds to a weak ordering over all items~\cite{shapley1974game}.
In this case there is a utility function representation for agent preferences, such that $i \succ j \iff u(i) > u(j)$, and $i \sim j \iff u(i) = u(j)$, where $u :\mathcal I \rightarrow \mathbb R$ is a continuous function.
We assume each agent has an underlying utility function, however in general we \emph{do not} assume preferences are transitive. 
In other words, we assume agents can rank items based on their relative value (represented by $u(\cdot)$), but in some cases they consider other factors in their response---causing them to be indecisive.
Next, to model indecision we propose mathematical representations of the causes for indecision from Section~\ref{sec:indecision}.

\subsection*{Mathematical Indecision Models}\label{sec:models-flip}

\input{sec_indecision_models}

\section*{Study 2: Fitting Indecision Models}\label{sec:study2}
In our second study, we aim to \emph{model} peoples' responses in the hypothetical kidney allocation scenario using indecision models from the previous section as well as standard preference models from the literature.
The models from the previous section can be used to predict peoples' responses, both when they are allowed to be indecisive, and when they are not.
To test both class of models, we conducted a survey with two groups of participants, where one group was were given the option to express indecision, and the other was not.
Each participant was assigned to 1 of the 150 random sequences, each of which contains 40 pairwise comparisons between two hypothetical kidney recipients with randomly generated values for age, number of dependents, and number of alcoholic drinks per week. We recruited 150 participants for group \flipgroup{}, which was given the option to express indecision\footnote{As in Study 1, this is phrased as ``flip a coin.''}. 18 participants were excluded from the analysis for failing attention checks, leaving us with a final sample of N=132. Another group, \noflipgroup{} (N=132), was recruited to respond to the same 132 sequences, but without the option to express indecision.

We remove 26 participants from \flipgroup{} who never express indecision, because it is not sensible to compare goodness-of-fit for different indecision models when the agent never chooses to be indecisive.
This study was reviewed and approved by our organization's Institutional Review Board; please see Appendix~\ref{app:survey-experiments} for a full description of the survey and dataset. 

\noindent\textbf{Model Fitting.}
In order to fit these indecision models to data, we assume that agent utility functions are linear: each item $i \in \mathcal I$ is represented by feature vector $\bm x^i \in \mathbb R^N$; agent utility for item $i$ is $u(i)=\bm u^\top \bm x^i$, where $\bm u \in \mathbb R^N$ is the agent's \emph{utility vector}.
We take a maximum likelihood estimation (MLE) approach to fitting each model: i.e., we select agent parameters $\bm u$ and $\lambda$ which maximize the log-likelihood (LL) of the training responses.
Since the LL of these models is not convex, we use random search via a Sobol process~\cite{sobol1967distribution}.
The search domain for utility vectors is $\bm u\in [-1, 1]^N$, the domain for probability parameters is $(0, 1)$, and the domain for $\lambda$ depends on the model type (see Appendix~\ref{app:model-fitting}).
The number of candidate parameters tested and the nature of the train-test split vary between experiments.
All code used for our analysis is available online, \footnote{\url{https://github.com/duncanmcelfresh/indecision-modeling}} and details of our implementation can be found in Appendix~\ref{app:model-fitting}.

We explore two different preference-modeling settings: learning individual indecision models, and learning group indecision models.

\subsection*{Individual Indecision Models}\label{sec:flip-experiments}
The indecision models from Section~\ref{sec:indecision-models} are indented to describe how an indecisive agent responds to queries---both when they are given the option to be indecisive, and when they are not.
Thus, we fit each of these models to responses from both participant groups: \flipgroup{} and \noflipgroup{}.
For each participant we randomly split their question-response pairs into a training and testing set of equal size (20 responses each).
For each participant we fit all five models from Section~\ref{sec:indecision-models}, and two baseline methods: \texttt{Rand} (express indecision with probability $q$ and chooses randomly between alternatives otherwise), \texttt{MLP} (a multilayer perceptron classifier with two hidden layers with 32 and 16 nodes).
%
%
We use \texttt{MLP} as a state-of-the-art benchmark, against which we compare our models; we use this benchmark to see how close our new models are to modern ML methods.

For group \flipgroup{} we estimate parameter $q$ for \texttt{NaiveRand} from the training queries; for \noflipgroup{} $q$ is $0$.
For \texttt{MLP} we train a classifier with one class for each response type, using scikit-learn~\cite{scikit-learn}: for \flipgroup{} responses we train a three-class model ($r\in \{0, 1, 2\}$), and for \noflipgroup{} we train a two-class model ($r\in \{1, 2\}$).

\begin{table*}
\begin{tabular}{@{}lrrrcccrrcc@{}}
\toprule
& \multicolumn{4}{c}{Group \flipgroup{} (both indecision and strict responses)} & & \multicolumn{4}{c}{Group \noflipgroup{} (only strict responses)}\\
\cmidrule(lr){2-5} \cmidrule(lr){7-10}
Model    & \#1st & \#2nd  & \#3rd  & Train/Test LL & & \# 1st & \# 2nd  & \# 3rd  & Train/Test LL \\
\midrule
\minD{}& 29 (27\%) & 23 (22\%)  & 13 (12\%) & -0.82/-0.85   &  & 26 (20\%) & 53 (40\%) & 34 (26\%) & -0.44/-0.47   \\
\maxD{}     & 11 (10\%) & 12 (11\%)  & 19 (18\%) & -0.81/-0.90   &  & 31 (23\%) & 57 (43\%) & 25 (19\%) & -0.44/-0.47   \\ 
\minL{}              & 8 (8\%)   & 32 (30\%)  & 17 (16\%) & -0.83/-0.88   &  & 1 (1\%)   & 5 (4\%)   & 20 (15\%) & -0.53/-0.56   \\
\maxL{}  & 22 (21\%) & 23 (22\%)  & 12 (11\%) & -0.81/-0.83   &  & 1 (1\%)   & 5 (4\%)   & 15 (11\%) & -0.53/-0.55   \\
\dom{}   & 0 (0\%)   & 3 (3\%)    & 9 (8\%)   & -0.88/-0.95   &  & 2 (2\%)   & 4 (3\%)   & 3 (2\%)   & -0.57/-0.58   \\
\midrule
\flipLogit{}    & 5 (5\%)   & 12 (11\%)  & 31 (29\%) & -0.84/-0.90   &  & 4 (3\%)   & 5 (4\%)   & 27 (20\%) & -0.53/-0.55   \\
\texttt{Rand} & 1 (1\%)   & 0 (0\%)    & 3 (3\%)   & -1.10/-1.10   &  & 6 (5\%)   & 0 (0\%)   & 1 (1\%)   & -0.69/-0.69   \\
\texttt{MLP}   & 30 (28\%) & 1 (1\%)    & 2 (2\%)   & -0.04/-1.15   &  & 61 (46\%) & 3 (2\%)   & 7 (5\%)   & -0.03/-0.49   \\
\bottomrule
\end{tabular}
\caption{\label{tab:individual-results}
Best-fit models for individual participants in group \flipgroup{} (left) and \noflipgroup{} (right).
The number of participants for which each model has the largest test log-likelihood (\#1st), second-largest test LL (\#2nd), as well as third-largest (\#3rd) are given for each model, and the median training and test LL over all participants.
}
\end{table*}

\noindent\textbf{Goodness-of-fit.} Using the standard ML approach, we select the best-fit models for each agent using the training-set LL, and evaluate the performance of these best-fit models using the test-set LL.
Table~\ref{tab:individual-results} shows the number of participants for which each model was the $1$st-, $2$nd-, and $3$rd best-fit for each participant (those with the greatest training-set LL), and the median test and train LL for each model.
First we observe that \emph{no indecision model} is a clear winner: several different models appear in the top 3 for each participant.
This suggests that different indecision models fit different individuals better than others --- there is not a single model that reflects everyone's choices.
However, some models perform better than others: \minD{} and \maxD{} appear often in the top 3 models, as does \maxL{} for group \flipgroup{}.

It it is somewhat surprising the \maxD{} fits participant responses, since this model does not seem intuitive: in \maxD{}, agents are indecisive when two alternatives have \emph{very different} utility---i.e. one has much greater utility than the other.
It is also surprising the \maxL{} is a good fit for group \flipgroup{}, but not for \noflipgroup{}.
One interpretation of this fact is that some people use (a version of) \maxL{} when they have the option, but they \emph{do not} use \maxL{} when indecision is not an option.
Another interpretation is that our modeling assumptions in Section~\ref{sec:models-strict} are wrong---however our dataset cannot definitively explain this discrepancy.

Finally, \texttt{MLP} is the most common best-fit model for all participants in both groups, though it is rarely a 2nd- or 3rd-best fit.
This suggests that the \texttt{MLP} benchmark accurately models \emph{some} participants' responses, and performs poorly for others; we expect this is due to overfitting.
%
%
While \texttt{MLP} is more accurate than our models in some cases, it does not shed light on why people are indecisive. 

\begin{figure}	
    %
	\begin{subfigure}[t]{0.47\textwidth}
		\centering
		\includegraphics[width=\linewidth]{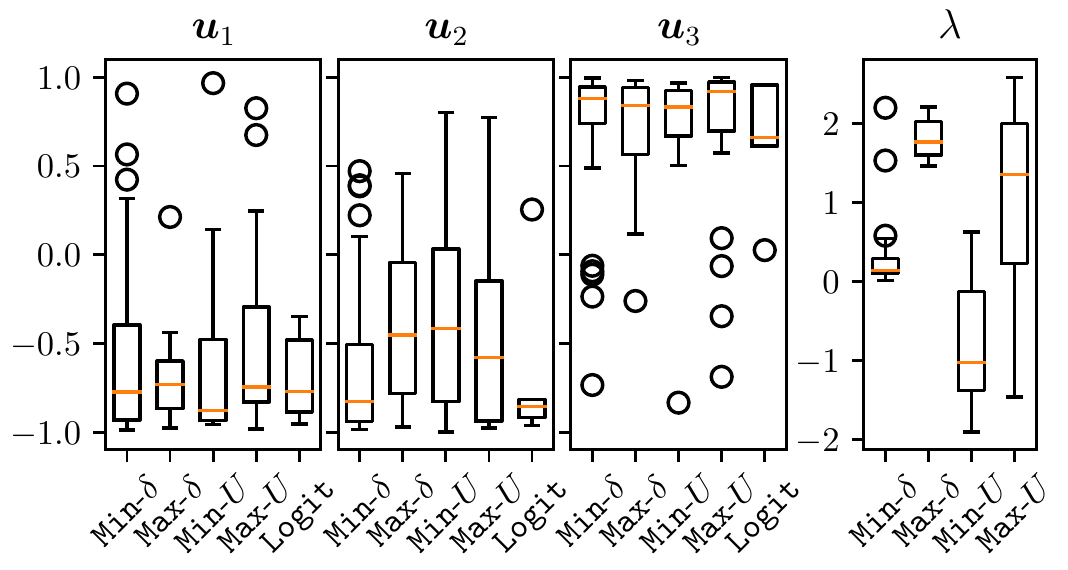}
		\caption{Participant best-fit model parameters for \flipgroup{}}
	\end{subfigure}
	\hfill
	%
	\begin{subfigure}[t]{0.47\textwidth}
	\centering
	\includegraphics[width=\linewidth]{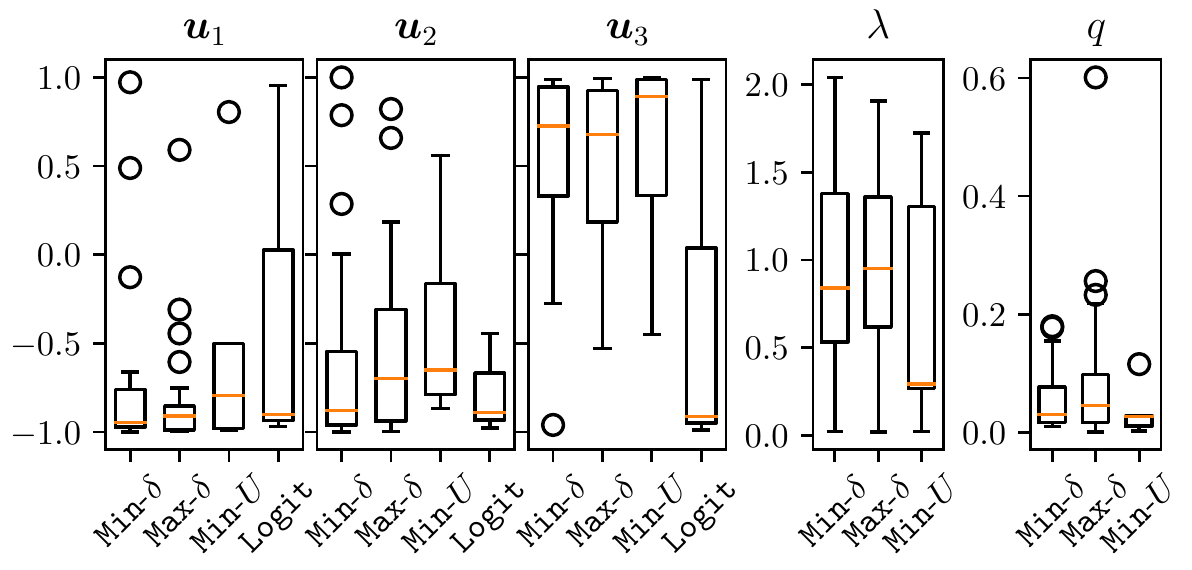}
		\caption{Participant best-fit model parameters for \noflipgroup{}}
	\end{subfigure}
	\caption{Best-fit parameters for each indecision model, for participants in group \flipgroup{} (top) and \noflipgroup{} (bottom).
	Elements of the agent utility vector correspond to patient age ($\bm u_1$), alcohol consumption ($\bm u_2$), and number of dependents ($\bm u_3$); the interpretation of $\lambda$ depends on the model class.
    Only participants for which the model is the 1st-best-fit are included (see Table 1).
	}\label{fig:u-individual}
\end{figure}

It is notable that some indecision models (\minD{} and \maxD{}) outperform the standard logit model (\flipLogit{}), both when they are learned from responses including indecision (group \flipgroup{}), and when they are learned from only strict responses (group \noflipgroup{}).
Thus, we believe that these indecision models give a more-accurate representation for peoples' decisions than the standard logit, both when they are given the option to be indecisive, and when they are not.

Since these indecision models may be accurate representations of peoples' choices, it is informative to examine the best-fit parameters.
Figure~\ref{fig:u-individual} shows best-fit parameters for participants in group \flipgroup{} (top) and \noflipgroup{} (bottom); for each indecision model, we show all learned parameters for participants for whom the model is the 1st-best-fit (see Table~\ref{tab:individual-results}).
Importantly, the best-fit values of $\bm u_1$, $\bm u_2$, and $\bm u_3$ are similar for all models, in both groups.
That is, \emph{in general}, people have similar relative valuations for different alternatives: $\bm u_1<0$ means younger patients are preferred over older patients, $\bm u_2<0$ means patients who consume less alcohol are preferred more; $\bm u_3>0$ means that patients with more dependents are preferred more.
We emphasize that the indecision model parameters for group \noflipgroup{} (bottom panel of Figure~\ref{fig:u-individual}) are learned using only strict responses.

These models are fit using only 20 samples, yet they provide useful insight into how people make decisions.
Importantly, our simple indecision models fit observed data better than the standard logit---both when people can express indecision, and when they cannot.
Thus, contrary to the common assumption in the literature, not all people are indecisive \emph{only} when two alternatives are nearly equivalent.
This assumption may be true for some people (participants for which \minD{} is a best-fit model), but it is not always true.

\subsection*{Group Models}\label{sec:group}
Next we turn to group decision models, where the goal is for an AI system to make decisions that reflect the values of a certain group of humans.
In the spirit of the social choice literature, we refer to agents as ``voters'', and suggested decisions as ``votes''.
We consider two distinct learning paradigms, where each reflects a potential use-case of an AI decision-making system.

The first paradigm, \emph{Population Modeling}, concerns a large or infinite number of voters; our goal is to estimate responses to new decision problems that are the {\em best} for the entire population.
This scenario is similar to conducting a national poll: we have a population including thousands or millions of voters, but we can only sample a small number (say, hundreds) of votes.
Thus, we aim to build a model that represents the entire population, using a small number of votes from a small number of voters.
There are several ways to aggregate uncertain voter models (see for example Chapter 10 of~\citet{handbook}); our approach is to estimate the next vote from a random voter in the population.
Since we cannot observe all voters, our model should generalize not only a ``known'' voter's future behavior, but \emph{all} voters' future behavior.

In the second paradigm, \emph{Representative Decisions}, we have a small number of ``representative'' voters; our goal is to estimate best responses to new decision problems for this group of representatives.
This scenario is similar to multi-stakeholder decisions including organ allocation or public policy design: these decisions are made by a small number of representatives (e.g., experts in medicine or policy), who often have very limited time to express their preferences.
As in \emph{Population Modeling} we aim to estimate the next vote from a random expert---however in this paradigm, all voters are ``known'', i.e., in the training data.

Both voting paradigms can be represented as a machine learning problem: observed votes are ``data'', with which we select a best-fit model from a hypothesis class; these models make predictions about future votes.\footnote{Several researchers have used techniques from machine learning for social choice~ \cite{doucette2015conventional,conitzer2017moral,kahng2019statistical,zhang2019pac}.}~
Thus, we split all observed votes into a training set (for model fitting) and a test set (for evaluation).
How we split votes into a training and test set is important: in \textit{Representative Decisions} we aim to predict future decisions from a \emph{known} pool of voters---so both the training and test set should contain votes from each voter.
In \textit{Population Modeling} we aim to predict future decisions from the entire voter population---so the training set should contain only some votes from some voters (i.e., ``training'' voters), while the test set should contain the remaining votes from training voters, and all responses from the non-training voters.

We propose several group indecision models, each of which is based on the models from Section~\ref{sec:indecision-models}; please see Appendix~\ref{app:model-fitting} for more details.

\noindent\textbf{\texttt{VMixture} Model.}
We first learn a best-fit indecision (sub)model for each training voter; the overall model generates responses by first selecting a training voter uniformly at random, and then responding according to their submodel.

\noindent\textbf{$k$-\texttt{Mixture} Model.}
This model consists of $k$ submodels, each of which is an indecision model with its own utility vector $\bm u$ and threshold $\lambda$.
The \emph{type} of each submodel (\texttt{Min}/\maxD{}, \texttt{Min}/\maxL{}, \dom{}) is itself a categorical variable.
Weight parameters $\bm w \in \mathbb R^k$ indicate the importance of each submodel. 
This model votes by selecting a submodel from the softmax distribution\footnote{With the softmax distribution, the probability of selecting $i$ is $e^{\bm w_i}/\sum_{j} e^{\bm w_j}$. We use this distribution for mathematical convenience, though it is straightforward to learn the distribution directly.}~
on $\bm w$, and responds according to the chosen submodel.

\noindent\textbf{$k$-\minD{} Mixture.}
This model is equivalent to $k$-\texttt{Mixture}, however all submodels are of type \minD{}. 
We include this model since \minD{} is the most-common best-fit indecision model for individual participants (see \S~\ref{sec:study2}).

We simulate both the \textit{Population Modeling} and \textit{Representative Decisions} settings using various train/test splits of our survey data.
For \textit{Population Modeling} we randomly select $100$ training voters; half of each training voter's responses are added to the test set, and half to the training set.
All responses from non-training voters are added to the test set.\footnote{Each voter in our data answers different questions, so all questions in the test set are ``new.''}

For \textit{Representative Decisions} we randomly select $20$ training voters (``representatives''), and randomly select half of each voter's responses for testing; all other responses are used for training; all non-training voters are ignored.

\input{table_2}

For both of these settings we fit all mixture models ($2$-\texttt{Mixture}, $2$-\minD{}, and \texttt{VMixture}), each individual indecision model from Section~\ref{sec:indecision-models}, and each each baseline model.
Table~\ref{tab:group-expert-population} shows the training-set and test-set LL for each method, for both voting paradigms.
Most indecision models achieve similar test-set LL, with the exception of \dom{}.
In the \textit{Representatives} setting, both mixture models and (non-mixture) indecision models perform well (notably, better than \texttt{MLP}.
This is somewhat expected, as the \textit{Representatives} setting uses very little training data, and complex ML approaches such as \texttt{MLP} are prone to overfitting---this is certainly the case in our experiments.
In the \textit{Population} setting the mixture models outperform individual indecision models; this is expected, as these mixture models have a strictly larger hypothesis class than any individual model.
Unsurprisingly, \texttt{MLP} achieves the greatest test-set LL in the \textit{Population} setting---yet provides no insight as to how these decisions are made.

\section*{Discussion}
In many cases it is natural to feel indecisive, for example when voting in an election or buying a new car; people are especially indecisive when their choices have moral consequences.
Importantly, there are many possible \emph{causes} for indecision, and each conveys different meaning: I may be indecisive when voting for a presidential candidate because I feel unqualified to vote; I may be indecisive when buying a car because all options seem too similar.
Using a small study, in Section~\ref{sec:pilot} we demonstrate that indecision cannot be interpreted as a ``flipping a coin'' to decide between alternatives.
This violates a key assumption in the technical literature, and it complicates the task of selecting the \emph{best} alternative for an individual or group.
Indeed, defining the ``best'' alternative for indecisive agents depends on what indecision means.

These philosophical and psychological questions have become critical to computer science researchers, since we now use preference modeling and social choice to guide deployed AI systems.
The indecision models we develop in Section~\ref{sec:indecision-models} and test in Section~\ref{sec:study2} provide a framework for understanding why people are indecisive---and how indecision may influence expressed preferences when people are allowed to be indecisive (\S~\ref{sec:models-flip}), and when they are required to express strict preferences (\S~\ref{sec:models-strict}).
The datasets collected in Study 1 (\S~\ref{sec:pilot}) and Study 2 (\S~\ref{sec:study2}) provide some insight into the causes for indecision, and we believe other researchers will uncover more insights from this data in the future.

Several questions remain for future work.
First, what are the causes for indecision, and what meaning do they convey? 
This question is well-studied in the philosophy and social science literature, and AI researchers would benefit from interdisciplinary collaboration. 
Methods for preference elicitation~\cite{blum2004preference} and active learning~\cite{freund1997selective} may be useful here.

Second, if indecision has meaning beyond the desire to ``flip a coin'', then what is the best outcome for an indecisive agent?
... for a group of indecisive agents?
This might be seen as a problem of winner determination, from a perspective of social choice~\cite{pini2011incompleteness}. \\ \\

\section*{Acknowledgements} Dickerson and McElfresh were supported in part by NSF CAREER IIS-1846237, NSF CCF-1852352, NSF D-ISN \#2039862, NIST MSE \#20126334, NIH R01 NLM-013039-01, DARPA GARD \#HR00112020007, DoD WHS \#HQ003420F0035, DARPA Disruptioneering (SI3-CMD) \#S4761, and a Google Faculty Research Award.  
Conitzer was supported in part by NSF IIS-1814056. This publication was made possible through the support of a grant (TWCF0321) from Templeton World Charity Foundation, Inc. to Conitzer, Schaich Borg, and Sinnott-Armstrong. The opinions expressed in this publication are those of the authors and do not necessarily reflect the views of Templeton World Charity Foundation, Inc.


\section*{Ethics Statement}
\input{ethics}

\bibliography{refs}

\clearpage
\appendix

\input{appendix_experiments}

\input{appendix_models}

\end{document}

%% file: pilot_study.tex
We first conduct a pilot study to illustrate the importance of measuring indecision.
Here we take the perspective of a preference-aggregator; we illustrate this perspective using a brief example: 
Suppose we must choose between two alternatives (X or Y), based on the preferences of several stakeholders.
Using a survey we ask all stakeholders to express a \emph{strict} preference (to ``vote'') for their preferred alternative; X receives 10 votes while Y receives 6 votes, so X wins.

Next we conduct the same survey, but allow stakeholders to vote for ``indecision'' instead; now, X receives 4 votes, Y receives 5 votes, and ``indecision'' receives 7 votes.
If we assume that voters are indecisive only when alternatives are nearly equivalent (assumption (A) from Section~\ref{sec:intro}), then each ``indecision'' vote is analogous to one half-vote for both X and Y, and therefore Y wins. In other words, in the first survey we assume that all indecisive voters choose randomly between X and Y.
However, if indecision has another meaning, then it is not clear whether X or Y wins. 
Thus, in order to make the best decision for our constituents we must understand what meaning is conveyed by indecisive voters. 
Unfortunately for our hypothetical decision-maker, assumption (A) is not always valid.

Using a small study, we test---and reject---assumption (A), which we frame as two different hypotheses, \textbf{H0-1}: \emph{if we discard all indecisive votes, then both X and Y receive the same proportion votes, whether or not indecision is allowed.} 
A second related hypothesis is 
\textbf{H0-2}: \emph{if we assign half of a vote to both X and Y when someone is indecisive, then both X and Y receive the same proportion votes, whether or not indecision is allowed.} 
We conducted the hypothetical surveys described above, using 15 kidney allocation questions (see Appendix~\ref{app:survey-experiments} for the survey text and analysis).
Participants were divided into two groups: participants in group \flipgroup{} (N=62) were allowed to express indecision (phrased as ``flip a coin to decide who receives the kidney''), while group \noflipgroup{} (N=60) was forced to choose one of the two recipients.
We test \textbf{H0-1} by identifying the \emph{majority patient}, ``X'' (who received the most votes) and the \emph{minority patient} ``Y'' for each of the 15 questions (details of this analysis are in Appendix~\ref{app:survey-experiments}). 
Overall, group \flipgroup{} cast 581 (74) votes for the majority (minority) patient, and 275 indecision votes; the \noflipgroup{} group cast 751 (149) votes for the majority (minority) patient.
Using a Pearson's chi-squared test we reject \textbf{H0-1} ($p<0.01$).
According to \textbf{H0-2}, we might assume that all indecision votes are ``effectively'' one half-vote for both the minority and majority patient.
In this case, the \flipgroup{} group casts 718.5 (211.5) ``effective'' votes for the majority (minority) patients; using these votes we reject \textbf{H0-2} ($p<0.01$).

In the context of our hypothetical choice between X and Y, this finding is troublesome: since we reject \textbf{H0-1} and \textbf{H0-2}, we cannot choose a winner by selecting the alternative with the most votes---or, if indecision is measured, the most ``effective'' votes.
If indecision has other meanings, then the ``best'' alternative depends on which meanings are used by each person; this is our focus in the remainder of this paper.

%% file: sec_indecision_models.tex
All models in this section are specified by two parameters: a utility function $u(\cdot)$ and a threshold $\lambda$. 
Each model is based on \emph{scoring functions}: when the agent observes a query they assign a numerical score to each response, and they respond with the response type that has maximal score; we assume that score ties are broken randomly, though this assumption will not be important.
In accordance with the literature, we assume the agent observes random iid additive error for each response score (see, e.g.,~\citet{soufiani2013preference}).
Let $S_r(i,j)$ be the agent's score for response $r$ to comparison $(i,j)$; the agent's response is given by
$$ 
R(i, j) = \argmax_{r \in \{0, 1, 2\}} S_r(i, j) + \epsilon_{rij}.
$$
That is, the agent has a deterministic score for each response $S_r(i, j)$, but when making a decision the agent observes a noisy version of this score, $S_r(i, j) + \epsilon_{rij}$.
We make the common assumption that noise terms $\epsilon_{rij}$ are iid Gumbel-distributed, with scale $\mu=1$.
In this case, the distribution of agent responses is 
\begin{equation}\label{eq:dist-flip}
    p(i, j, r) = \frac{e^{S_r(i, j)}}{e^{ S_0(i, j)} + e^{ S_1(i, j)} + e^{ S_2(i, j)}}.
\end{equation}
Each indecision model is defined using different score functions $S_r(\cdot, \cdot)$.
Score functions for strict responses are always symmetric, in the sense that $S_{2}(i,j)=S_1(j,i)$; thus we need only define $S_1(\cdot, \cdot)$ and $S_0(\cdot, \cdot)$.
We group each model by their cause for indecision from Section~\ref{sec:indecision}.

\noindent\textbf{Difference-Based Models: \minD{}, \maxD{}}
Agents are indecisive when the utility difference between alternatives is either smaller than threshold $\lambda$ (\minD{}) or greater than $\lambda$ (\maxD{}). 
The score functions for these models are
\begin{equation*}
\begin{array}{rl}
    \text{\minD{}}: &\left\{ \begin{array}{rl}
         S_1(i, j) &\equiv u(i) - u(j) \\
         S_0(i, j) &\equiv \lambda 
    \end{array} \right.\\

        \text{\maxD{}}: & \left\{\begin{array}{rl}
         S_1(i, j) &\equiv  u(i) - u(j)  \\
         S_0(i, j) &\equiv 2|u(i) - u(j)| - \lambda
    \end{array}\right. \\
    \end{array}
\end{equation*}
Here $\lambda$ should be non-negative: for example with \minD{}, $\lambda\leq 0$ means the agent is never indecisive, while for \maxD{} this means the agent is always indecisive. 
Model \maxD{} seems counter-intuitive (if one alternative is clearly better than the other, why be indecisive?), yet we include it for completeness.
Note that this is only one example of a difference-based model: instead the agent might assess alternatives using a distance measure $d : \mathcal I \times \mathcal I \rightarrow \mathbb R_+$, rather than $u(\cdot)$.

\noindent\textbf{Desirability-Based Models: \minL{}, \maxL{}}
Agents are indecisive when the utility of \emph{both} alternatives is below threshold $\lambda$ (\minL{}), or when the utility of both alternatives is greater than $\lambda$ (\maxL{}).
Unlike the difference-based models, $\lambda$ here may be positive or negative.
The score functions for these models are
\begin{equation*}
\begin{array}{rl}
    \text{\minL{}}: &\left\{ \begin{array}{rl}
         S_1(i,j)&\equiv u(i)  \\
         S_0(i,j) &\equiv \lambda
    \end{array} \right.\\

        \text{\maxL{}}: & \left\{
        \begin{array}{rl}
         S_1(i, j) &\equiv  u(i) \\
         S_0(i, j) &\equiv 2 \min\{u(i),u(j)\} - \lambda
    \end{array}
        \right. \\
    \end{array}
\end{equation*}
%
%
Both of these models motivated in the literature (see \S~\ref{sec:indecision}).

\noindent\textbf{Conflict-Based Model: \dom{}}
In this model the agent is indecisive unless one alternative \emph{dominates} the other in all features, by threshold at least $\lambda$.
For this indecision model, we need a utility measure associated with each feature of each item; for this purpose, let $u_n(i)$ be the utility associated with feature $n$ of item $i$.
As before, $\lambda$ here may be positive or negative.
The score functions for this model are
\begin{equation*}
\begin{array}{rl}
    \text{\dom{}}: &\left\{ \begin{array}{rl}
    S_1(i,j) &\equiv \min_{n \in [N]}\left(u_n(i) - u_n(j)\right) \\
    S_0(i, j) &\equiv \lambda 
    \end{array} \right.
    \end{array}
\end{equation*}
This is one example of a conflict-based indecision model, though we might imagine others. 

These models serve as a class of hypotheses which describe how agents respond to comparisons when they are allowed to be indecisive.
Using the response distribution in (\ref{eq:dist-flip}), we can assess how well each model fits with an agent's (possibly indecisive) responses. 
However, in many cases agents are \emph{required} to express strict preferences---they are not allowed to be indecisive (as in Section~\ref{sec:pilot}).
With slight modification the score-based models from this section can be used even when agents are forced to express \emph{only} strict preferences; we discuss this in the next section.

\subsection*{Indecision Models for Strict Comparisons}\label{sec:models-strict}
We assume that agents may prefer to be indecisive, even when they are required to express strict preferences.
That is, we assume that agents use an underlying \emph{indecision} model to express \emph{strict} preferences.
When they cannot express indecision, we assume that they either \emph{resample} from their decision distribution, or they choose randomly.
That is, we assume agents use a two-stage process to respond to queries: first they sample a response $r$ from their response distribution $p(\cdot, \cdot, r)$; if $r$ is strict ($1$ or $2$), then they express it, and we are done.
If they sample indecision ($0$), then they flip a weighted coin to decide how to respond: 
\begin{enumerate}[leftmargin=1.2cm]
    \item[(heads)] with probability $q$ they re-sample from their response distribution until they sample a strict response, without flipping the weighted coin again
    \item[(tails)] with probability $1-q$ they choose uniformly at randomly between responses $1$ and $2$. 
\end{enumerate}
That is, they respond according to distribution 
\begin{align}\label{eq:noflip-dist}
p_{strict}(i, j, r) \equiv \begin{cases} 
\begin{array}{l} q \left( \frac{e^{S(i, j)}+ (1/2)e^{S_0(i, j)}}{C} \right) \\+ \frac{1-q}{D} \left( e^{S_1(i, j)} \right)  \end{array} & \text{if} \,\, r=1 \\
\begin{array}{l} q \left( \frac{e^{S_2(i, j)}+ (1/2)e^{S_0(i, j)}}{C} \right) \\+ \frac{1-q}{D}e^{S_2(i, j)}   \end{array}  & \text{if} \,\, r=2 
\end{cases}
\end{align}
Here, $C \equiv e^{S_0(i, j)} + e^{S_1(i, j)}+ e^{S_2(i, j)}$ , and $D\equiv e^{S_1(i, j)}   + e^{S(_2(i, j)}$.
The (heads) condition from above has another interpretation: the agent chooses to sample from a ``strict'' logit, induced by only the score functions for strict responses, $S_1(i, j)$ and $S_1(i, j)$.
We discuss this model in more detail, and provide an intuitive example, in Appendix~\ref{app:model-fitting}.

We now have mathematical indecision models which describe how indecisive agents respond to comparison queries, both when they are allowed to express indecision (\S~\ref{sec:models-flip}), and when they are not (\S~\ref{sec:models-strict}).
The model in this section, and response distributions (\ref{eq:dist-flip}) and (\ref{eq:noflip-dist}), represent one way indecisive agents might respond when they are forced to express strict preferences.
The question remains whether any of these models accurately represent peoples' expressed preferences in real decision scenarios.
In the next section we conduct a second, larger survey to address this question.

%% file: table_2.tex
\setlength{\tabcolsep}{4pt}

\begin{table}
{\footnotesize
\begin{tabular}{@{}lcccc@{}}
\toprule
Model Name                               & \multicolumn{2}{c}{\textit{Represenatitives} (20)} & \multicolumn{2}{c}{\textit{Population} (100)} \\
\cmidrule(lr){2-3}\cmidrule(lr){4-5}
                                         & \flipgroup{}&  \noflipgroup{} &   \flipgroup{}&   \noflipgroup{}\\
                                         \midrule
2-\minD{} & -0.90/-0.88      & -0.46/-0.47      & -0.87/-0.88    & -0.54/-0.52         \\
2-\texttt{Mixture}  & -0.87/\textbf{-0.86}      & -0.45/-0.47     & -0.87/-0.88    & -0.53/-0.52       \\
\texttt{VMixture} & -0.92/-0.90      & -0.49/-0.51   & -0.93/-0.94   & -0.57/-0.56          \\
\midrule
\minD{} & -0.92/-0.90      & -0.46/-0.48    & -0.87/-0.87   & -0.54/-0.53          \\
\maxD{}  & -0.95/-0.90      & -0.45/\textbf{-0.46}   & -0.96/-0.95    & -0.54/-0.52           \\
\minL{}& -0.96/-0.95      & -0.52/-0.54     & -0.98/-0.99  & -0.58/-0.57         \\
\maxL{}   & -0.87/\textbf{-0.86}      & -0.54/-0.54      & -0.94/-0.94  & -0.58/-0.57         \\
\dom{}    & -1.08/-1.07      & -0.57/-0.58    & -1.05/-1.06   & -0.61/-0.60         \\
\midrule
\texttt{MLP}& -0.40/-1.55      & -0.15/-0.85     & -0.71/\textbf{-0.77}  & -0.42/\textbf{-0.51}         \\
\flipLogit{}   & -0.91/-0.88      & -0.53/-0.54     & -0.93/-0.94   & -0.57/-0.56       \\
\texttt{Rand}   & -1.03/-1.00      & N/A             & -1.07/-1.07   & N/A                 \\
\bottomrule
\end{tabular}
}
\caption{\label{tab:group-expert-population} 
Average train-set and test-set LL per question (reported as ``train/test'') for \emph{Representative Decisions} with 20 training voters,  (left) and \emph{Population Modeling} with 100 training voters (right), for both the \flipgroup{} and \noflipgroup{} participant groups.
The greatest test-set LL is highlighted for each column.
For \emph{Representatives}, the test set includes only votes from the representative voters; for \textit{Population}, the test set includes all voters.
}
\end{table}

%% file: ethics.tex
Many AI systems are designed and constructed with the goal of promoting the interests, values, and preferences of users and stakeholders who are affected by the AI systems. Such systems are deployed to make or guide important decisions in a wide variety of contexts, including medicine, law, business, transportation, and the military. When these systems go wrong, they can cause irreparable harm and injustice. There is, thus, a strong moral imperative to determine which AI systems best serve the interests, values, and preferences of those who are or might be affected.

To satisfy this imperative, designers of AI systems need to know what affected parties really want and value. Most surveys and experiments that attempt to answer this question study decisions between two options without giving participants any chance to refuse to decide or to adopt a random method, such as flipping a coin. Our studies show that these common methods are inadequate, because providing this third option—which we call indecision—changes the preferences that participants express in their behavior. Our results also suggest that people often decide to use a random method for variety of reasons captured by the models we studied. Thus, we need to use these more complex methods—that is, to allow indecision—in order to discover and design AI systems to serve what people really value and see as morally permitted. That lesson is the first ethical implication of our research.

Our paper also teaches important ethical lessons regarding justice in data collection. It has been shown that biases can, and are, introduced at the level of data collection.  Our results open the door to the suggestion that biases could be introduced when a participant’s values are elicited under the assumption of a strict preference. Consider a simple case of choosing between two potential kidney recipients, A and B, who are identical in all aspects, except A has 3 drinks a week while B has 4. Throughout our studies, we have consistently observed that  participants would overwhelmingly give the kidney to patient A who has 1 fewer drink each week, when forced to choose between them. However, when given the option to do so, most would rather flip a coin. An argument can be made here that the data collection mechanism under the strict-preference assumption is biased against patient B and others who drink more than average.

Finally, our studies also have significant relevance to randomness as a means of achieving fairness in algorithms. As our participants were asked to make moral decisions regarding who should get the kidney, one interpretation of their decisions to flip a coin is that the fair thing to do is often to flip a coin so that they (and humans in general) do not have to make an arbitrary decision. The modeling techniques proposed here differ from the approach to fairness that conceives random decisions as guaranteeing equity in the distribution of resources. Our findings about model fit suggest that humans sometimes employ random methods largely in order to avoid making a difficult decision (and perhaps also in order to avoid personal responsibility). If our techniques are applied to additional problems, they will further the discussion of algorithmic fairness by emphasizing the role of randomness and indecision. This advance can improve the ability of AI systems to serve their purposes within moral constraints.

\noindent\textbf{Experiment Scenario: Organ Allocation.} 
Our experiments focus on a hypothetical scenario involving the allocation of scarce donor organs. 
We use organ allocation since it is a real, ethically-fraught problem, which often involves AI or other algorithmic guidance.
However our hypothetical organ allocation, and our survey experiments, are not intended to reflect the many ethical and logistical challenges of organ transplantation; these issues are settled by medical experts and policymakers.
Our experiments do not focus on a realistic organ allocation scenario, and our results should not be interpreted as guidance for transplantation policy.

%% file: appendix_experiments.tex
\section{Online Survey Experiments}\label{app:survey-experiments}

This appendix describes our survey experiments in greater detail. 
~\ref{app:platform} describes the online platform we used for this survey, Section~\ref{sec:app-study-1} describes Study 1 and our analysis, and Section~\ref{sec:app-study-2} describes the design of Study 2.

\subsection{Online Platform}\label{app:platform}

All online experiments were conducted using a custom online survey platform.
After agreeing to an online consent form, participants were shown background information on kidney allocation and about the patient features in this survey, shown below:
\begin{displayquote}
Sometimes people with certain diseases or injuries require a kidney transplant. If they don't have a biologically compatible friend or family member who is willing to donate a kidney to them, they must wait to receive a kidney from a stranger.

Choose which of two patients should receive a sole available kidney. Information about Patient A will always be on the left. Information about Patient B will always be on the right. The characteristics of each patient will change in each trial. Patients who do receive the kidney will undergo an operation that is almost always successful. Patients who do not receive the kidney will remain on dialysis and are likely to die within a year.
\end{displayquote}

After completing an online consent form, participants were asked to respond to a series of comparisons between two potential kidney recipients.
Each recipient is represented by three features: ``number of child dependent(s)'', ``years old'', and ``drinks per day prediagnosis.'' 
Figure~\ref{fig:screenshot} shows a screenshot of the decision scenario.
\begin{figure}[h]
    \centering
    \includegraphics[width=0.45\textwidth]{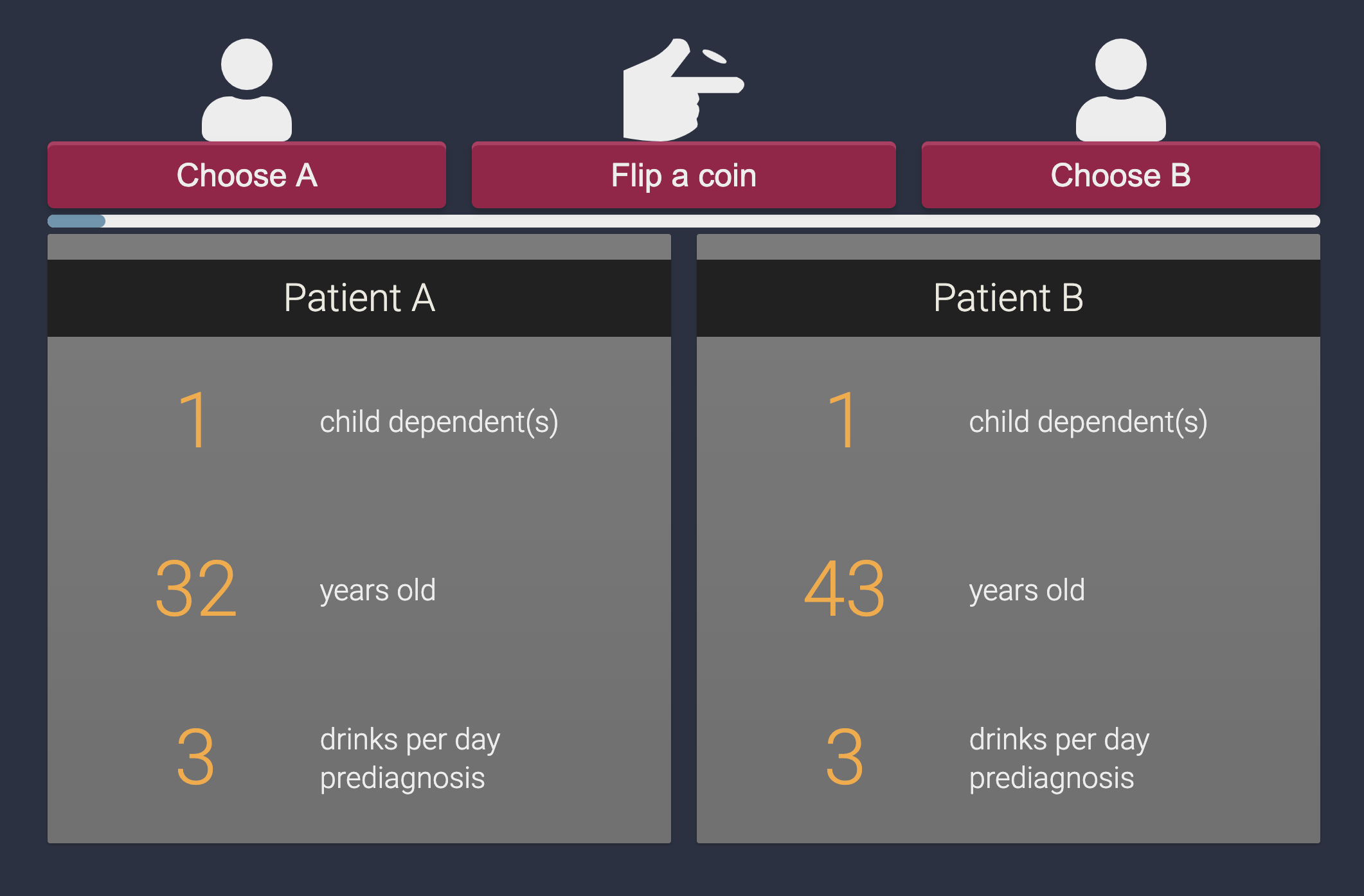}
    \caption{\label{fig:screenshot} Screenshot of a comparison question from our online survey (Study 1). This screenshot is for the group \flipgroup{}; for participants in group \noflipgroup{}, the middle response option ``Flip a coin'' was not shown. }
    \label{fig:my_label}
\end{figure}

All participants were recruited on Amazon Mechanical Turk\footnote{\url{https://www.mturk.com/}} (MTurk). We included only participants in the United States, who have completed more than 500 HITs, with HIT approval rate at least $98\%$, and who have not participated in any previous studies for this project.

\subsection{Study 1}\label{sec:app-study-1}

We recruited 120 participants via MTurk.
One participant was excluded from the cohort due to incompleteness, leaving us a sample of N= 119 (32\% female and 68\% male; mean age = 35.2, SD = 10.12, 82\% white) with N=60 for group \flipgroup{} and N=59 for group \noflipgroup{}.
On our online platform, both groups were asked to make decisions on a set of 15 pairs of hypothetical patients, whose features were pre-determined \emph{a priori}.
Both groups were given the same sequence of scenarios; the features of each patient in these scenarios is included in our dataset (included in the supplement and online, see below).

The \emph{Indecisive} group were given the additional option to flip a coin, instead of choosing one of the two patients. 

Anonymized responses from Study 1 are available online,\footnote{Link removed for blind review.} and included in this paper's online supplement.

\paragraph{Study 1 Analysis: Hypothesis Testing}
For this analysis we refer to each strict response as a ``vote''.
For example if a participant expresses the preference for patient $A$ over patient $B$, we say this is a vote for $A$.
To test hypotheses \textbf{H0-1} and \textbf{H0-2} we first identify the \emph{majority patient} (the patient who received more votes than the other patient); the other patient is referred to as the \emph{minority patient}.
Coincidentally, the majority and minority patients were the same for both groups, \flipgroup{} and \noflipgroup{}.
Table~\ref{tab:pilot-details} shows the number of votes for the minority and majority patient for each question, for both groups.

\begin{table}[]
    \centering
    \begin{tabular}{@{}lcccccc@{}}
\toprule
\multirow{2}{*}{Q\#} & \multicolumn{3}{c}{Group \flipgroup{}} & \multicolumn{2}{c}{Group \noflipgroup{}} \\
\cmidrule(lr){2-4}\cmidrule(lr){5-6}
& \#Maj. & \#Min. & \#Flip & \#Maj. & \#Min.\\
\midrule
1  & 31                             & 5                              & 2        & 38 & 22 \\
2  & 48                             & 2                              & 12     & 50 & 10 \\
3  & 43                             & 2                              & 17     & 57 & 3  \\
4  & 40                             & 13                             & 9      & 42 & 18 \\
5  & 37                             & 0                              & 25     & 51 & 9  \\
6  & 55                             & 0                              & 7      & 57 & 3  \\
7  & 43                             & 1                              & 18     & 56 & 4  \\
8  & 37                             & 9                              & 16     & 48 & 12 \\
9  & 29                             & 8                              & 25     & 43 & 17 \\
10 & 22                             & 5                              & 35     & 54 & 6  \\
11 & 41                             & 12                             & 9      & 43 & 17 \\
12 & 51                             & 3                              & 8      & 54 & 6  \\
13 & 29                             & 4                              & 29     & 55 & 5  \\
14 & 42                             & 1                              & 19     & 56 & 4  \\
15 & 33                             & 9                              & 20     & 47 & 13 \\
\bottomrule
\end{tabular}
\caption{Number of votes for the \emph{majority patient} (\#Maj.) and \emph{minority patient} (\#Min.) for each group. The number of ``flip a coin'' votes (\#Flip) is shown for group \flipgroup{}. The right column Q\# indicates the order in which the comparison was shown to each participant.\label{tab:pilot-details}}
\end{table}

\subsection{Study 2}\label{sec:app-study-2}

We first recruited 150 participants using MTurk for the \flipgroup{} group. 
Each participant was assigned a randomly generated sequence of 40 pairs of hypothetical patients, and they were presented with the option to either give a kidney to one of the patients, or flip a coin (see Figure~\ref{fig:screenshot}). 
Patient features were generated uniformly at random  from the ranges:
\begin{itemize}
    \item \# dependents: 0, 1, 2
    \item age: 25, \dots, 70 
    \item \# drinks: 1, 2, 3, 4, 5
\end{itemize}
In addition, 3 or 4 attention-check pairs, in which the participant is presented with the choice between an already deceased patient and a ``favorable'' patient,\footnote{A 30-year-old patient who consumed 1 alcoholic drink per week, with 2 dependents.} were randomly distributed in each sequence. 
After data collection, 18 participants were excluded for failing at least one attention check, i.e., choosing to give the kidney to the deceased patient. 
This leaves us N=132 participants (age distribution was 31\%: 18-29, 48\%: 39-30, 10\%: 40-49, 6\%: 50-59, 3\%: 60+; gender distribution was 29\%: female, 70\%: male; racial distribution was 75\%: white, 25\% nonwhite).

Next we recruited 153 participants for group \noflipgroup{}; these participants were given the exact same task as the \flipgroup{} group, but without the option to flip a coin. 
21 participants were excluded from the analysis due to attention check failures, leaving us with a final sample of N=132 (age distribution was 26\%: 18-29, 46\%: 39-30, 17\%: 40-49, 10\%: 50-59, 2\%: 60+; gender distribution was 36\%: female,  63\%: male; racial distribution was 72\%: white, 28\% nonwhite).

Anonymized responses from Study 2 are available online,\footnote{Link removed for blind review.} and included in this paper's online supplement.

%% file: appendix_models.tex
\section{Fitting Indecision Models}\label{app:model-fitting}

In this appendix we provide additional details on the indecision models from Section~\ref{sec:indecision-models}, as well as details of our computational experiments.

First, in~\ref{app:models} we motivate the score-based decision models (from Section~\ref{sec:indecision-models}) using an intuitive---and equivalent---representation as \emph{response functions}.
In~\ref{app:strict} we provide additional motivation for the \emph{strict} response model from Section~\ref{sec:models-strict}. 
In~\ref{app:group} we provide additional details on our group indecision models.
Finally, in~\ref{app:implementation} we describe the implementation of our computational experiments.

\subsection{Response Functions vs Score-Based Models}\label{app:models}

In the score-based models from Section~\ref{sec:indecision-models}, the agent responds by evaluating a ``score'' for each possible response.
Here we provide an intuitive motivation for each of these indecision models, framed as response \emph{functions}.
As in Section~\ref{sec:indecision-models}, an agent response function $R:\mathcal I \times \mathcal I \rightarrow \{0, 1, 2\}$ maps a pair of items (a comparison question) to a response.
In this section, all agent response functions are expressed in terms of the agent utility function $u(\cdot)$ and threshold $\lambda$.
Each response function identifies a set of \emph{feasible} responses for the agent, which depend on the agent utility function and threshold.
If there are multiple feasible responses, the agent chooses one uniformly at random.
Importantly, we show below that these response functions are identical to the score-based response functions for models in Section~\ref{sec:indecision-models}, when the agent observes no ``noise.''

Below we formalize each response function, grouped by by their ``causes'' (see  Section~\ref{sec:indecision}).

Each of the functions here appears
We emphasize that each of these response ``functions'' is in fact a multifunction, as multiple responses may be possible. However 

\noindent\textbf{Difference-Based Response Functions: \minD{}, \maxD{}}
Here the agent is indecisive when the utility difference between alternatives is either smaller than threshold $\lambda$ (\minD{}) or greater than $\lambda\in \mathbb R_+$ (\maxD{}). 
The corresponding response functions are
\begin{equation*}
\begin{array}{rl}
    \text{\minD{}}: &
            R(i, j) \equiv \begin{cases}
1 & \text{if}\, u(i)- u(j) \geq \lambda \\
    2 &\text{if}\, u(i) - u(j)  \leq \lambda \\
0 &\text{if}\, |u(i) - u(j)| \leq \lambda \\
\end{cases} \\
        \text{\maxD{}}: & 
        R(i, j) \equiv \begin{cases}
1 & \text{if}\, 0 \leq u(i)- u(j) \leq \lambda  \\
    2 &\text{if}\, -\lambda  \leq  u(i) - u(j) \leq 0 \\
0 &\text{if}\, |u(i) - u(j)| \geq \lambda  \\
\end{cases}\\
    \end{array}
\end{equation*}
Note that in these definitions, multiple responses may be feasible (i.e., the conditions may be met for multiple responses.
In this case, we assume the agent selects a feasible response uniformly at random.
For example, for both models \minD{} and \maxD{}, if $u(i) - u(j) = \lambda$ then the agent selects a response randomly with either $1$ or $0$.

In these models the agent response depends on the utility difference between $i$ and $j$, ($u(i) - u(j)$). 
Depending on how this utility difference compares with threshold $\lambda$, the agent may be indecisive.
Since the agent is indecisive only when the absolute difference in item utility ($|u(i) - u(j)|$) is too large or too small, negative $\lambda$ is not meaningful here---thus, we only consider $\lambda>0$.

%
%

\noindent\textbf{Desirability-Based Models: \minL{}, \maxL{}}
Here the agent is indecisive when the utility of \emph{both} alternatives is below threshold $\lambda\in \mathbb R$ (\minL{}), or when the utility of both alternatives is greater than $\lambda$ (\maxL{}).
Unlike the difference-based models, $\lambda$ here may be positive or negative.
The response functions for these models are
\begin{equation*}
\begin{array}{rl}
    \text{\minL{}}: &
    R(i, j) \equiv \begin{cases}
1 & \text{if}\, u(i) \geq \max\{ u(j), \lambda \} \\
   2&\text{if}\,\, u(j) \geq \max\{ u(i), \lambda \}\\ 
0 &\text{if}\, \lambda \geq \max\{u(i), u(j) \}
\end{cases}\\
        \text{\maxL{}}: & 
        R(i, j) \equiv \begin{cases}
1 & \text{if}\, u(j) \leq \min\{u(i), \lambda\} \\
   2 &\text{if}\,\, u(i) \leq \min\{u(j), \lambda\} \\
0 &\text{if}\, \lambda \leq \min\{u(i), u(j)\} \\ 
\end{cases}\\
    \end{array}
\end{equation*}
As before, if there are multiple feasible responses, the agent selects one feasible response uniformly at random. 

Unlike the difference-based models, both positive and negative $\lambda$ are reasonable here.
For example: suppose an agent is only indecisive when both alternatives are very undesireable (e.g., both items have utility less than $-100$).
This agent's decisions might be best modeled by \minL{}, with $\lambda=-100$. 

\noindent\textbf{Conflict-Based Model: \dom{}}
Here the agent is indecisive unless one alternative \emph{dominates} the other in all features, by threshold at least $\lambda\in \mathbb R$.
For this indecision model, we need a utility measure associated with each feature of each item; for this purpose, let $u_n(i)$ be the utility associated with feature $n$ of item $i$.
As before we assume $\lambda$ may be positive or negative.
The response function for this model is
\begin{equation*}
\begin{array}{rl}
    \text{\dom{}}: &
    R(i,j) \equiv \begin{cases}
1 & \text{if}\,\, M_{ij} \geq \max\{ M_{ji},\lambda \}\\
    2 &\text{if}\,\, M_{ji} \geq \max\{ M_{ij},\lambda \} \\
0 &\text{if}\,\, \lambda \geq \max\{ M_{ij},M_{ji} \}\\
\end{cases}
    \end{array}
\end{equation*}
where $M_{ij} \equiv \min\limits_{n=1, \dots, N}\{u_n(i) - u_n(j) \}$ and $M_{ji} \equiv \min\limits_{n=1, \dots, N}\{u_n(j) - u_n(i) \}$.
In other words, $M_{ij}$ is the \emph{minimum} difference between the feature utilities of $i$ and $j$: if $M_{ij}$ is positive, then all features of alternative $i$ are strictly better than those of $j$.
If neither $i$ nor $j$ ``dominates'' the other by at least threshold $\lambda$, then the agent is indecisive.
As before, the agent selects uniformly at random from all feasible responses.

While these response functions appear qualitatively different from the score functions in Section~\ref{sec:indecision-models}, they are in fact identical under certain circumstances.
\begin{prop}
For each indecision model (\minD{}, \maxD{}, \minL{}, \maxL{}, \dom{}), the response function given in Appendix~\ref{app:models} is identical to the response function induced by score functions $S_0(\cdot, \cdot)$ and $S_1(\cdot, \cdot)$ as in Section~\ref{sec:indecision-models}, when the agent observes no score error.
This score-induced response function is expressed as
$$
R^S(i,j) \equiv \argmax\limits_{r\in\{0, 1, 2\}} S_r(i, j)
$$
where if multiple scores are maximal (i.e., the corresponding response is \emph{feasible}), the agent selects a response with maximal score uniformly at random.
\end{prop}
\begin{proof}
We prove equivalence for each indecision model separately.
Note that, if both response functions $R^S(i,j)$ and $R(i,j)$ have the same set of \emph{feasible} responses for a given comparison $(i,j)$, then these responses are identical--since both response function chooses a feasible response uniformly at random.
Thus, we prove that the set of \emph{feasible responses} is the same for both $R^S(i,j)$ and $R(i,j)$, for an arbitrary comparison $(i,j)$.

\paragraph{\minD{}}
For score-based response function $R^S(i,j)$, response $1$ is feasible if the following conditions are met
$$
\begin{array}{l} 
S_1(i, j) \geq  S_0(i, j) \\
S_1(i, j) \geq S_2(i, j)
\end{array} \iff 
\begin{array}{l} 
u(i) - u(j) \geq \lambda \\
u(i) - u(j) \geq 0
\end{array} 
$$
where the left and right side are equivalent.
Note that the right side conditions are equivalent to the conditions for response $1$ in $R(i, j)$, since $\lambda$ is positive.
Note that the same argument holds for response $2$.

Next, for score-based response function $R^S(i,j)$, response $0$ is feasible if the following conditions are met
$$
\begin{array}{l} 
S_0(i, j) \geq  S_1(i, j) \\
S_0(i, j) \geq S_2(i, j)
\end{array} \iff 
\begin{array}{l} 
\lambda \geq u(i) - u(j) \\
\lambda \geq u(j) - u(i) 
\end{array} 
$$
and these conditions are equivalent to $|u(i) - u(j)| \leq \lambda$, since $\lambda$ is positive. 
This condition is equivalent to the condition for response $0$ in $R(i, j)$.

\paragraph{\maxD{}}
For score-based response function $R^S(i,j)$, response $1$ is feasible if the following conditions are met
$$
\begin{array}{l} 
S_1(i, j) \geq  S_0(i, j) \\
S_1(i, j) \geq S_2(i, j)
\end{array} \iff 
\begin{array}{l} 
u(i) - u(j) \\
\quad \geq 2|u(i) - u(j)| - \lambda \\
 \\
u(i) - u(j) \geq 0.
\end{array} 
$$
Note that the first constraint right side reduces to $u(i) - u(i) \leq \lambda$; thus, these conditions are equivalent to the conditions for response $1$ in $R(i, j)$.
Note that the same argument holds for response $2$.

Next, for score-based response function $R^S(i,j)$, response $0$ is feasible if the following conditions are met
{\small
$$
\begin{array}{l} 
S_0(i, j) \geq  S_1(i, j) \\
S_0(i, j) \geq S_2(i, j)
\end{array} \iff 
\begin{array}{l} 
2|u(i) - u(j)| - \lambda  \geq u(i) - u(j) \\
2|u(i) - u(j)| - \lambda  \geq u(j) - u(i).
\end{array} 
$$
}
There are two cases: (1) if $u(i) \geq u(j)$, then the first condition on the right side reduces to $|u(i) - u(j) \geq \lambda$, and the second condition on the right side holds trivially; (2) if $u(i) < u(j)$, then the second condition on the right side reduces to $|u(i) - u(j) \geq \lambda$, and the first condition on the right side holds trivially.
In both cases, these conditions are equivalent to the conditions for response $0$ in $R(i, j)$.

\paragraph{\minL{}}
For score-based response function $R^S(i,j)$, response $1$ is feasible if the following conditions are met
$$
\begin{array}{l} 
S_1(i, j) \geq  S_0(i, j) \\
S_1(i, j) \geq S_2(i, j)
\end{array} \iff 
\begin{array}{l} 
u(i) \geq \lambda \\
u(i)  \geq u(j)
\end{array} 
$$
where the right-side conditions reduce to $u(i) \geq \max\{u(i), \lambda\}$, which is equivalent to the condition for response $1$ in $R(i, j)$.
Note that the same argument holds for response $2$.

Next, for score-based response function $R^S(i,j)$, response $0$ is feasible if the following conditions are met
$$
\begin{array}{l} 
S_0(i, j) \geq  S_1(i, j) \\
S_0(i, j) \geq S_2(i, j)
\end{array} \iff 
\begin{array}{l} 
\lambda \geq u(i) \\
\lambda  \geq u(j)
\end{array} 
$$
which is equivalent to $\lambda \geq \max\{u(i), u(j)\}$, the condition for response $0$ in $R(i, j)$.

\paragraph{\maxL{}}
For score-based response function $R^S(i,j)$, response $1$ is feasible if the following conditions are met

$$
\begin{array}{l} 
S_1(i, j) \geq  S_0(i, j) \\
S_1(i, j) \geq S_2(i, j)
\end{array} \iff 
\begin{array}{l} 
u(i) - u(j) \\
 \quad\geq \max\{u(i), u(j)\} - \lambda \\
 \\
u(i)  \geq u(j).
\end{array} 
$$
The first condition on the right side reduces to $u(j) \leq \lambda$; thus, the right side conditions are equivalent to $u(j) \leq \min\{u(i), \lambda\}$, which is the condition for response $1$ in function $R(i, j)$.
Note that the same argument holds for response $2$.

Next, for score-based response function $R^S(i,j)$, response $0$ is feasible if the following conditions are met
$$
\begin{array}{l} 
S_0(i, j) \geq  S_1(i, j) \\
S_0(i, j) \geq S_2(i, j)
\end{array} \iff 
\begin{array}{l} 
\max\{u(i), u(j)\} - \lambda \\
\quad \geq u(i) - u(j) \\
 \\
\max\{u(i), u(j)\} - \lambda \\
 \quad \geq u(j) - u(i).
\end{array} 
$$

There are two cases: (1) if $u(i) \geq u(j)$, then the first condition on the right side reduces to $u(j) \geq \lambda$, and the second condition on the right side reduces to $2u(j) - u(j) \geq \lambda$ (which holds trivially); (2) if $u(i) < u(j)$, then the second condition reduces to $u(i) \geq \lambda$ (and the first condition holds trivially).
In both cases, these conditions are equivalent to $\lambda \leq \min\{ u(i), u(j)\}$, which is the condition for response $0$ in $R(i, j)$.

\paragraph{\dom{}}
This proof is identical to that of \minL{}: let $u(i)$ and $u(j)$ be replaced by $M_{ij}$ and $M_{ji}$, respectively, and the proof is identical.

\end{proof}

\subsection{Strict Decision Models}\label{app:strict}

In Section~\ref{sec:models-strict} we describe how indecision models can be used to model scenarios where an indecisive agent is \emph{required} to express a strict preference.
Here we assume that the agent uses a two-step process to respond, represented in Figure~\ref{fig:strict-decision}.

\tikzset{
    desicion/.style={
        diamond,
        draw,
        text width=3em,
        text badly centered,
        inner sep=0pt
    },
    block/.style={
        rectangle,
        draw,
        text width=10em,
        text centered,
        rounded corners
    },
    finalblock/.style={
        rectangle,
        draw,
        text width=10em,
        text centered,
    },
    cloud/.style={
        draw,
        ellipse,
        minimum height=2em
    },
    descr/.style={
        fill=white,
        inner sep=2.5pt
    },
    connector/.style={
     -latex,
     font=\scriptsize
    },
    rectangle connector/.style={
        connector,
        to path={(\tikztostart) -- ++(#1,0pt) \tikztonodes |- (\tikztotarget) },
        pos=0.5
    },
    rectangle connector/.default=-2cm,
    straight connector/.style={
        connector,
        to path=--(\tikztotarget) \tikztonodes
    }
}

\begin{figure}[h]
    \centering
   
\begin{tikzpicture}[node distance=1cm]
\centering
\node (m1) [block] {Indecisive agent is presented with comparison $(i, j)$};
\node (m2) [block,below of=m1,node distance=2cm] {Agent draws response $r$ from their response distribution};
\node (m3)  [block,below left of=m2, node distance=3cm] {Agent flips a coin with ``heads'' probability $q$};
\node (m4)  [finalblock,below right of=m2, node distance=3cm] {Agent responds $r$};

\draw[thick,->] (m1) -- (m2);
\draw [thick,->] (m2) -- node[descr] {$r\in \{1, 2\}$} (m4);
\draw [thick,->] (m2) -- node[descr] {$r=0$} (m3);

\node (m5)  [finalblock,below of=m3, node distance=2.5cm] {Agent draws from strict distribution $p^S(i, j)$};
\node (m6)  [finalblock,right of=m5, node distance=4cm] {Agent responds either $1$ or $2$ uniformly at random};

\draw [thick,->] (m3) -- node[descr] {heads} (m5);
\draw [thick,->] (m3) -- node[descr] {tails} (m6);

\end{tikzpicture}
\caption{\label{fig:strict-decision} Flowchart describing our model for an indecisive agent who is required to express a strict preference.}
\end{figure}
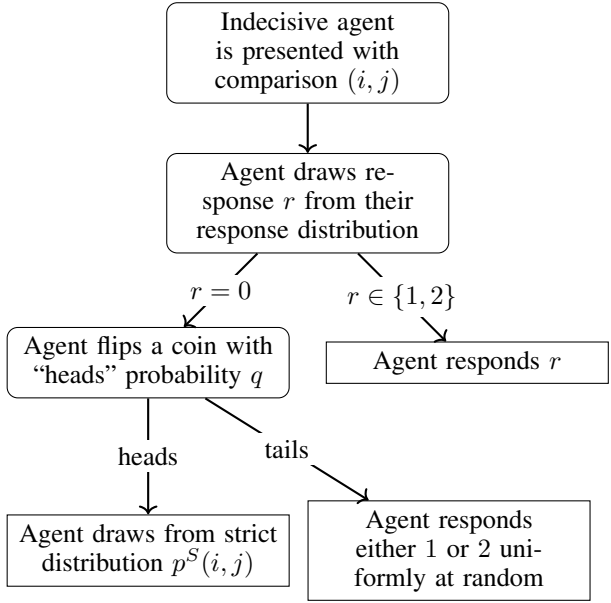

If the agent's coin flip is ``heads'' (with probability $q$), then the agent draws from a \emph{strict} version of their response distribution, defined as
$$ p^S(i, j, r) \equiv \frac{e^{S_r(i,j)}}{e^{S_1(i,j)} + e^{S_2(i,j)}}$$
for $r\in \{0, 1\}$. 
Note that this is similar to the agent's \emph{true} response distribution (Equation~\ref{eq:dist-flip}), but assigns zero probability to response $0$.

The overall response distribution described in Figure~\ref{fig:strict-decision} has a closed-form expression, since the probability-$q$ coin flip is independent from each draw of the agent's decision function.
As stated in Section~\ref{sec:models-strict}, this distribution is
\begin{align*}
p_{strict}(i, j, r) \equiv \begin{cases} 
\begin{array}{l} q \left( \frac{e^{S(i, j)}+ (1/2)e^{S_0(i, j)}}{C} \right) \\+ \frac{1-q}{D} \left( e^{S_1(i, j)} \right)  \end{array} & \text{if} \,\, r=1 \\
\begin{array}{l} q \left( \frac{e^{S_2(i, j)}+ (1/2)e^{S_0(i, j)}}{C} \right) \\+ \frac{1-q}{D}e^{S_2(i, j)}   \end{array}  & \text{if} \,\, r=2 
\end{cases}
\end{align*}
where, $C \equiv e^{S_0(i, j)} + e^{S_1(i, j)}+ e^{S_2(i, j)}$ , and $D\equiv e^{S_1(i, j)}   + e^{S(_2(i, j)}$.
The (heads) condition from above has another interpretation: the agent chooses to sample from a ``strict'' logit, induced by only the score functions for strict responses, $S_1(i, j)$ and $S_1(i, j)$.
We discuss this model in more detail, and provide an intuitive example, in Appendix~\ref{app:model-fitting}.

\subsection{Group Decision Models}\label{app:group}

Here we outline the mathematical group decision models from Section~\ref{sec:group}.

A set of $L$ observed responses is represented by vectors $\bm i\in \mathcal I^L$, $\bm j \in \mathcal I^L$, $\bm r\in \{0, 1, 2\}^L$, where $\bm i_k$ and $\bm j_k$ are the indices of items $i$ and $j$ in query $j$, and $\bm r_k$ is the observed agent's response.

\paragraph{\texttt{VMixture}}
This model is parameterized only by the best-fit models for each of its constituent voters.
Let $V\in \mathbb Z$ be the number of voters, and let $S^v_r(\cdot, \cdot)$ be the best-fit score function for voter $v$ and response $r$. 
Since we take an MLE approach, the goodness-of-fit metric for these models is the log-likelihood of the model, given observed responses.

The log-likelihood for model \texttt{VMixture} is 

\begin{equation*}
\sum_{l=1}^L \log \left(\sum_{v=1}^V \frac{1}{V} p^{v}(\bm i_l,\bm j_l, \bm r_l) \right)
\end{equation*}
where 
$$ p^v(i, j, r) \equiv \frac{e^{S^v_r(i, j)}}{e^{ S^v_0(i, j)} + e^{ S^v_1(i, j)} + e^{ S^v_2(i, j)}}
$$
is the response distribution for voter $v$.

\paragraph{$k$-\texttt{Mixture}}
This model class is parameterized by $k$ distinct sets of submodel parameters: each submodel consists of a utility vector $\bm u\in \mathbb R^N$ and threshold $\lambda\in \mathbb R$; the \emph{type} of each model is also a variable (i.e., a categorical variable).
Weight parameters $\bm w$ indicate the importance of each submodel.
Let $S^k_r(\cdot, \cdot)$ be the score function for model $l\in \{1, \dots, k\}$ and response $r\in \{0, 1, 2\}$; these score functions depend on the type of each model (see Section~\ref{sec:indecision-models}).
For the $k$-\texttt{Mixture} model, the log-likelihood is
\begin{equation*}
\sum_{l=1}^L \log \left(\sum_{k'=1}^k \frac{e^{\bm w_{k'}}}{\sum_{n=1}^k e^{\bm w_n}} p^{k'}(\bm i_l,\bm j_l, \bm r_l) \right)
\end{equation*}
where 
$$ p^{k'}(i, j, r) \equiv \frac{e^{S^{k'}_r(i, j)}}{e^{ S^{k'}_0(i, j)} + e^{ S^{k'}_1(i, j)} + e^{ S^{k'}_2(i, j)}}
$$
is the response distribution for model $k'$.

\subsection{Experiments and Implementation}\label{app:implementation}
All code used for our computational experiments is available online,\footnote{\url{https://github.com/duncanmcelfresh/indecision-modeling}} and attached in our supplementary material.
All code is written in Python 3.7, and uses packages Ax\footnote{\url{https://ax.dev/}} for random sampling.
All experiments were run on a single Intel Xeon E5-2690 node with 16GB memory.

For all experiments, models were fit by sampling several random parameter sets using a Sobol process (implemented using Ax).
Each model is ``trained'' using a different number or random Sobol points in our experiments:
\begin{itemize}
    \item Individual indecision models (Table~\ref{tab:individual-results}): 1,000 points for \flipgroup{}, and 5,000 for \noflipgroup{} (which uses an additional parameter $q$).
    \item Group indecision models (Table~\ref{tab:group-expert-population}, models \minD{}, \maxD{}, \minL{}, \maxL{}, \dom{}, \flipLogit{}): 5,000 points 
    \item \texttt{VMixture}: 500 points for group \flipgroup{} and 1,000 points for \noflipgroup{}, for each individual model.
    \item $k$-\texttt{Mixture}, $k$-\minD{}: 100,000 points
\end{itemize}

%% file: main.bbl
\begin{thebibliography}{36}
\providecommand{\natexlab}[1]{#1}
\providecommand{\url}[1]{\texttt{#1}}
\providecommand{\urlprefix}{URL }
\expandafter\ifx\csname urlstyle\endcsname\relax
  \providecommand{\doi}[1]{doi:\discretionary{}{}{}#1}\else
  \providecommand{\doi}{doi:\discretionary{}{}{}\begingroup
  \urlstyle{rm}\Url}\fi

\bibitem[{Agarwal et~al.(2019)Agarwal, Ashlagi, Rees, Somaini, and
  Waldinger}]{agarwal2019empirical}
Agarwal, N.; Ashlagi, I.; Rees, M.~A.; Somaini, P.~J.; and Waldinger, D.~C.
  2019.
\newblock Equilibrium Allocations under Alternative Waitlist Designs: Evidence
  from Deceased Donor Kidneys.
\newblock Working Paper 25607, National Bureau of Economic Research.
\newblock \doi{10.3386/w25607}.
\newblock \urlprefix\url{http://www.nber.org/papers/w25607}.

\bibitem[{Blackburn(1996)}]{Blackburn1996-BLADDP}
Blackburn, S. 1996.
\newblock Dilemmas: Dithering, Plumping, and Grief.
\newblock In Mason, H.~E., ed., \emph{Moral Dilemmas and Moral Theory}, 127.
  Oxford University Press.

\bibitem[{Blum et~al.(2004)Blum, Jackson, Sandholm, and
  Zinkevich}]{blum2004preference}
Blum, A.; Jackson, J.; Sandholm, T.; and Zinkevich, M. 2004.
\newblock Preference elicitation and query learning.
\newblock \emph{Journal of Machine Learning Research} 5(Jun): 649--667.

\bibitem[{Brandt et~al.(2016)Brandt, Conitzer, Endriss, Lang, and
  Procaccia}]{handbook}
Brandt, F.; Conitzer, V.; Endriss, U.; Lang, J.; and Procaccia, A.~D. 2016.
\newblock \emph{Handbook of computational social choice}.
\newblock Cambridge University Press.

\bibitem[{Chang(2002)}]{Chang2002-px}
Chang, R. 2002.
\newblock The Possibility of Parity.
\newblock \emph{Ethics} 112(4): 659--688.

\bibitem[{Conitzer et~al.(2017)Conitzer, Sinnott-Armstrong, Borg, Deng, and
  Kramer}]{conitzer2017moral}
Conitzer, V.; Sinnott-Armstrong, W.; Borg, J.~S.; Deng, Y.; and Kramer, M.
  2017.
\newblock Moral decision making frameworks for artificial intelligence.
\newblock In \emph{Proceedings of the Thirty-First AAAI Conference on
  Artificial Intelligence}, 4831--4835.

\bibitem[{Day et~al.(2012)Day, Bateman, Carson, Dupont, Louviere, Morimoto,
  Scarpa, and Wang}]{day2012ordering}
Day, B.; Bateman, I.~J.; Carson, R.~T.; Dupont, D.; Louviere, J.~J.; Morimoto,
  S.; Scarpa, R.; and Wang, P. 2012.
\newblock Ordering effects and choice set awareness in repeat-response stated
  preference studies.
\newblock \emph{Journal of environmental economics and management} 63(1):
  73--91.

\bibitem[{DeShazo and Fermo(2002)}]{deshazo2002designing}
DeShazo, J.; and Fermo, G. 2002.
\newblock Designing choice sets for stated preference methods: the effects of
  complexity on choice consistency.
\newblock \emph{Journal of Environmental Economics and management} 44(1):
  123--143.

\bibitem[{Donagan(1984)}]{Donagan1984-DONCIR}
Donagan, A. 1984.
\newblock Consistency in Rationalist Moral Systems.
\newblock \emph{Journal of Philosophy} 81(6): 291--309.
\newblock \doi{jphil198481650}.

\bibitem[{Doucette, Larson, and Cohen(2015)}]{doucette2015conventional}
Doucette, J.~A.; Larson, K.; and Cohen, R. 2015.
\newblock Conventional Machine Learning for Social Choice.
\newblock In \emph{AAAI}, 858--864.

\bibitem[{Freedman et~al.(2020)Freedman, Borg, Sinnott-Armstrong, Dickerson,
  and Conitzer}]{freedman2020adapting}
Freedman, R.; Borg, J.~S.; Sinnott-Armstrong, W.; Dickerson, J.~P.; and
  Conitzer, V. 2020.
\newblock Adapting a kidney exchange algorithm to align with human values.
\newblock \emph{Artificial Intelligence} 103261.

\bibitem[{Freund et~al.(1997)Freund, Seung, Shamir, and
  Tishby}]{freund1997selective}
Freund, Y.; Seung, H.~S.; Shamir, E.; and Tishby, N. 1997.
\newblock Selective sampling using the query by committee algorithm.
\newblock \emph{Machine learning} 28(2-3): 133--168.

\bibitem[{Furnham, Simmons, and McClelland(2000)}]{furnham2000decisions}
Furnham, A.; Simmons, K.; and McClelland, A. 2000.
\newblock Decisions concerning the allocation of scarce medical resources.
\newblock \emph{Journal of Social Behavior and Personality} 15(2): 185.

\bibitem[{Furnham, Thomson, and McClelland(2002)}]{furnham2002allocation}
Furnham, A.; Thomson, K.; and McClelland, A. 2002.
\newblock The allocation of scarce medical resources across medical conditions.
\newblock \emph{Psychology and Psychotherapy: Theory, Research and Practice}
  75(2): 189--203.

\bibitem[{Gangemi and Mancini(2013)}]{gangemi2013moral}
Gangemi, A.; and Mancini, F. 2013.
\newblock Moral choices: the influence of the do not play god principle.
\newblock In \emph{Proceedings of the 35th Annual Meeting of the Cognitive
  Science Society, Cooperative Minds: Social Interaction and Group Dynamics},
  2973--2977. Cognitive Science Society, Austin, TX.

\bibitem[{Gerasimou(2018)}]{gerasimou2018indecisiveness}
Gerasimou, G. 2018.
\newblock Indecisiveness, undesirability and overload revealed through rational
  choice deferral.
\newblock \emph{The Economic Journal} 128(614): 2450--2479.

\bibitem[{Hare(1981)}]{Hare1981-HARMTI}
Hare, R.~M. 1981.
\newblock \emph{Moral Thinking: Its Levels, Method, and Point}.
\newblock Oxford: Oxford University Press.

\bibitem[{Kahng et~al.(2019)Kahng, Lee, Noothigattu, Procaccia, and
  Psomas}]{kahng2019statistical}
Kahng, A.; Lee, M.~K.; Noothigattu, R.; Procaccia, A.; and Psomas, C.-A. 2019.
\newblock Statistical foundations of virtual democracy.
\newblock In \emph{International Conference on Machine Learning}, 3173--3182.

\bibitem[{Luce(1998)}]{Luce1998-je}
Luce, M.~F. 1998.
\newblock Choosing to Avoid: Coping with Negatively {Emotion-Laden} Consumer
  Decisions.
\newblock \emph{Journal of Consumer Research} 24(4): 409--433.

\bibitem[{Mattei, Saffidine, and Walsh(2018)}]{mattei2018fairness}
Mattei, N.; Saffidine, A.; and Walsh, T. 2018.
\newblock Fairness in deceased organ matching.
\newblock In \emph{Proceedings of the 2018 AAAI/ACM Conference on AI, Ethics,
  and Society}, 236--242.

\bibitem[{McElfresh and Dickerson(2018)}]{mcelfresh2018balancing}
McElfresh, D.~C.; and Dickerson, J.~P. 2018.
\newblock Balancing Lexicographic Fairness and a Utilitarian Objective with
  Application to Kidney Exchange.
\newblock In \emph{Proceedings of the AAAI Conference on Artificial
  Intelligence}, volume~32.

\bibitem[{McIntyre(1990)}]{McIntyre1990-MACMD}
McIntyre, A. 1990.
\newblock Moral Dilemmas.
\newblock \emph{Philosophy and Phenomenological Research} 50(n/a): 367--382.
\newblock \doi{10.2307/2108048}.

\bibitem[{Mochon(2013)}]{mochon2013single}
Mochon, D. 2013.
\newblock Single-option aversion.
\newblock \emph{Journal of Consumer Research} 40(3): 555--566.

\bibitem[{Noothigattu et~al.(2018)Noothigattu, Gaikwad, Awad, Dsouza, Rahwan,
  Ravikumar, and Procaccia}]{noothigattu2018voting}
Noothigattu, R.; Gaikwad, S.; Awad, E.; Dsouza, S.; Rahwan, I.; Ravikumar, P.;
  and Procaccia, A. 2018.
\newblock A Voting-Based System for Ethical Decision Making.
\newblock \emph{AAAI 2018} .

\bibitem[{Oedingen, Bartling, and Krauth(2018)}]{oedingen2018public}
Oedingen, C.; Bartling, T.; and Krauth, C. 2018.
\newblock Public, medical professionals’ and patients’ preferences for the
  allocation of donor organs for transplantation: study protocol for discrete
  choice experiments.
\newblock \emph{BMJ open} 8(10): e026040.

\bibitem[{Pedregosa et~al.(2011)Pedregosa, Varoquaux, Gramfort, Michel,
  Thirion, Grisel, Blondel, Prettenhofer, Weiss, Dubourg, Vanderplas, Passos,
  Cournapeau, Brucher, Perrot, and Duchesnay}]{scikit-learn}
Pedregosa, F.; Varoquaux, G.; Gramfort, A.; Michel, V.; Thirion, B.; Grisel,
  O.; Blondel, M.; Prettenhofer, P.; Weiss, R.; Dubourg, V.; Vanderplas, J.;
  Passos, A.; Cournapeau, D.; Brucher, M.; Perrot, M.; and Duchesnay, E. 2011.
\newblock Scikit-learn: Machine Learning in {P}ython.
\newblock \emph{Journal of Machine Learning Research} 12: 2825--2830.

\bibitem[{Pini et~al.(2011)Pini, Rossi, Venable, and
  Walsh}]{pini2011incompleteness}
Pini, M.~S.; Rossi, F.; Venable, K.~B.; and Walsh, T. 2011.
\newblock Incompleteness and incomparability in preference aggregation:
  Complexity results.
\newblock \emph{Artificial Intelligence} 175(7-8): 1272--1289.

\bibitem[{Railton(1992)}]{Railton1992-RAIPDA}
Railton, P. 1992.
\newblock Pluralism, Determinacy, and Dilemma.
\newblock \emph{Ethics} 102(4): 720--742.
\newblock \doi{10.1086/293445}.

\bibitem[{Scheunemann and White(2011)}]{scheunemann2011ethics}
Scheunemann, L.~P.; and White, D.~B. 2011.
\newblock The ethics and reality of rationing in medicine.
\newblock \emph{Chest} 140(6): 1625--1632.

\bibitem[{Shapley and Shubik(1974)}]{shapley1974game}
Shapley, L.~S.; and Shubik, M. 1974.
\newblock \emph{Game Theory in Economics: Chapter 4, Preferences and Utility}.
\newblock Santa Monica, CA: RAND Corporation.

\bibitem[{Sinnott{-}Armstrong(1988)}]{Sinnott-Armstrong1988-SINMD}
Sinnott{-}Armstrong, W. 1988.
\newblock \emph{Moral Dilemmas}.
\newblock Blackwell.

\bibitem[{Sobol'(1967)}]{sobol1967distribution}
Sobol', I.~M. 1967.
\newblock On the distribution of points in a cube and the approximate
  evaluation of integrals.
\newblock \emph{Zhurnal Vychislitel'noi Matematiki i Matematicheskoi Fiziki}
  7(4): 784--802.

\bibitem[{Soufiani, Parkes, and Xia(2013)}]{soufiani2013preference}
Soufiani, H.~A.; Parkes, D.~C.; and Xia, L. 2013.
\newblock Preference elicitation for General Random Utility Models.
\newblock In \emph{Proceedings of the Twenty-Ninth Conference on Uncertainty in
  Artificial Intelligence}, 596--605.

\bibitem[{Tversky and Shafir(1992)}]{tversky1992choice}
Tversky, A.; and Shafir, E. 1992.
\newblock Choice under conflict: The dynamics of deferred decision.
\newblock \emph{Psychological science} 3(6): 358--361.

\bibitem[{Zakay(1984)}]{zakay1984choose}
Zakay, D. 1984.
\newblock " To choose or not to choose": On choice strategy in face of a single
  alternative.
\newblock \emph{The American journal of psychology} 373--389.

\bibitem[{Zhang and Conitzer(2019)}]{zhang2019pac}
Zhang, H.; and Conitzer, V. 2019.
\newblock A PAC framework for aggregating agents’ judgments.
\newblock In \emph{Proceedings of the AAAI Conference on Artificial
  Intelligence}, volume~33, 2237--2244.

\end{thebibliography}
